\documentclass[letterpaper, 10 pt, conference]{ieeeconf}
\IEEEoverridecommandlockouts 
\usepackage{amsmath} 
\usepackage{amsfonts}
\usepackage{mathtools} 
\usepackage{amssymb}
\usepackage{graphicx}
\usepackage{bm}
\usepackage[colorinlistoftodos, textsize=footnotesize]{todonotes}
\usepackage[colorlinks=true, allcolors=black]{hyperref}
\usepackage{setspace}
\usepackage{algorithm}
\usepackage{algpseudocode}
\usepackage{tikz}
\usepackage{verbatim}
\usepackage{xcolor}
\usepackage{subcaption}
\usepackage{todonotes}
\usepackage{makecell}
\usepackage{adjustbox}
\usepackage{mathrsfs}
\usepackage{booktabs}
\usepackage[normalem]{ulem}
\usepackage[thinc]{esdiff}
\usepackage{algorithmicx}
\usepackage{hyperref}
\newtheorem{theorem}{Theorem}
\newtheorem{lemma}{Lemma}

\newif\ifcompile

\DeclareMathOperator*{\argmin}{arg\,min}

\newcommand{\vf}[1]{{\bm{#1}}}
\newcommand{\mf}[1]{{\mathbf{#1}}}

\newcommand{\bgamma}{{\bm\gamma}}

\title{\LARGE \bf
A Convex Formulation of Material Points and Rigid Bodies with GPU-Accelerated Async-Coupling for Interactive Simulation}

\author{Chang Yu$^{*1}$, Wenxin Du$^{*1}$, Zeshun Zong$^{1}$, Alejandro Castro$^{2}$, Chenfanfu Jiang$^{1}$, Xuchen Han$^{2}$
\thanks{* equal contribution.}
\thanks{$^{1}${\tt\footnotesize \{changyu1,wenxindu,zeshunzong,cffjiang\}@ucla.edu}, AIVC Laboratory, UCLA, USA.}
\thanks{$^{2}${\tt\footnotesize \{alejandro.castro,xuchen.han\}@tri.global}, Toyota Research Institute, USA.}
}

\begin{document}

\maketitle
\thispagestyle{empty}
\pagestyle{empty}

\begin{abstract}
We present a novel convex formulation that weakly couples the Material Point Method (MPM) with rigid body dynamics through frictional contact, optimized for efficient GPU parallelization. Our approach features an asynchronous time-splitting scheme to integrate MPM and rigid body dynamics under different time step sizes. We develop a globally convergent quasi-Newton solver tailored for massive parallelization, achieving up to 500× speedup over previous convex formulations without sacrificing stability. Our method enables interactive-rate simulations of robotic manipulation tasks with diverse deformable objects including granular materials and cloth, with strong convergence guarantees. We detail key implementation strategies to maximize performance and validate our approach through rigorous experiments, demonstrating superior speed, accuracy, and stability compared to state-of-the-art MPM simulators for robotics. We make our method available in the open-source robotics toolkit, Drake.
\end{abstract}
\section{INTRODUCTION}

With the rise of robotics foundation models \cite{bib:firoozi2023foundation}, simulation has become an indispensable tool for both policy training and evaluation~\cite{bib:choi2021use, bib:zhang2024vlabench}. Simulations enable generating large-scale training data at a fraction of the cost of real-world collection and facilitate evaluating model checkpoints in controlled environments, ensuring repeatability. In these settings, modern physics simulators for robotics must satisfy several key requirements to be effective: (i) strong numerical stability guarantees, (ii) sufficient accuracy to capture real-world success and failure modes, (iii) computational efficiency to support data collection and policy evaluation at scale, and (iv) the ability to model diverse environments.

Among existing physics simulators for robotics, Drake \cite{bib:drake} and MuJoCo \cite{bib:mujoco} are well-regarded for their strong robustness guarantees and high accuracy required for manipulation tasks. Both employ a convex formulation of frictional contact, ensuring stability and global convergence \cite{bib:todorov2014, bib:castro2022unconstrained}. However, despite recent advancements in deformable body simulation within the convex framework \cite{bib:mujoco3, bib:han2023convex, bib:zong2024convex}, these simulators struggle to efficiently simulate deformable bodies with a large number of degrees of freedom (DoFs) at interactive rates. This limitation arises partly from the need to solve poorly-conditioned convex optimization problems that necessitates direct linear solvers with $O(n^3)$ complexity \cite{bib:castro2022unconstrained}. The inclusion of deformable bodies exacerbates this challenge, as it significantly increases the number of DoFs—often by orders of magnitude—making efficient solutions increasingly difficult. Moreover, the inherently serial nature of direct linear solvers precludes efficient parallelization on GPUs, further restricting scalability.

Among various methods for simulating deformable materials, the Material Point Method (MPM) has gained traction due to its ability to handle large deformations and topology changes \cite{bib:gu2023maniskill2}. However, existing simulators struggle to achieve both robustness and efficiency \cite{bib:zong2024convex}. In this paper, we propose a novel convex formulation for coupling MPM with rigid bodies through frictional contact, designed for efficient GPU parallelization. We prove the stability and global convergence of our frictional contact model. Furthermore, we demonstrate that our method efficiently simulates a wide range of materials while preserving the high accuracy required for robotic manipulation tasks.

\begin{figure}[t]
\centerline{\includegraphics[width=1.0\columnwidth]{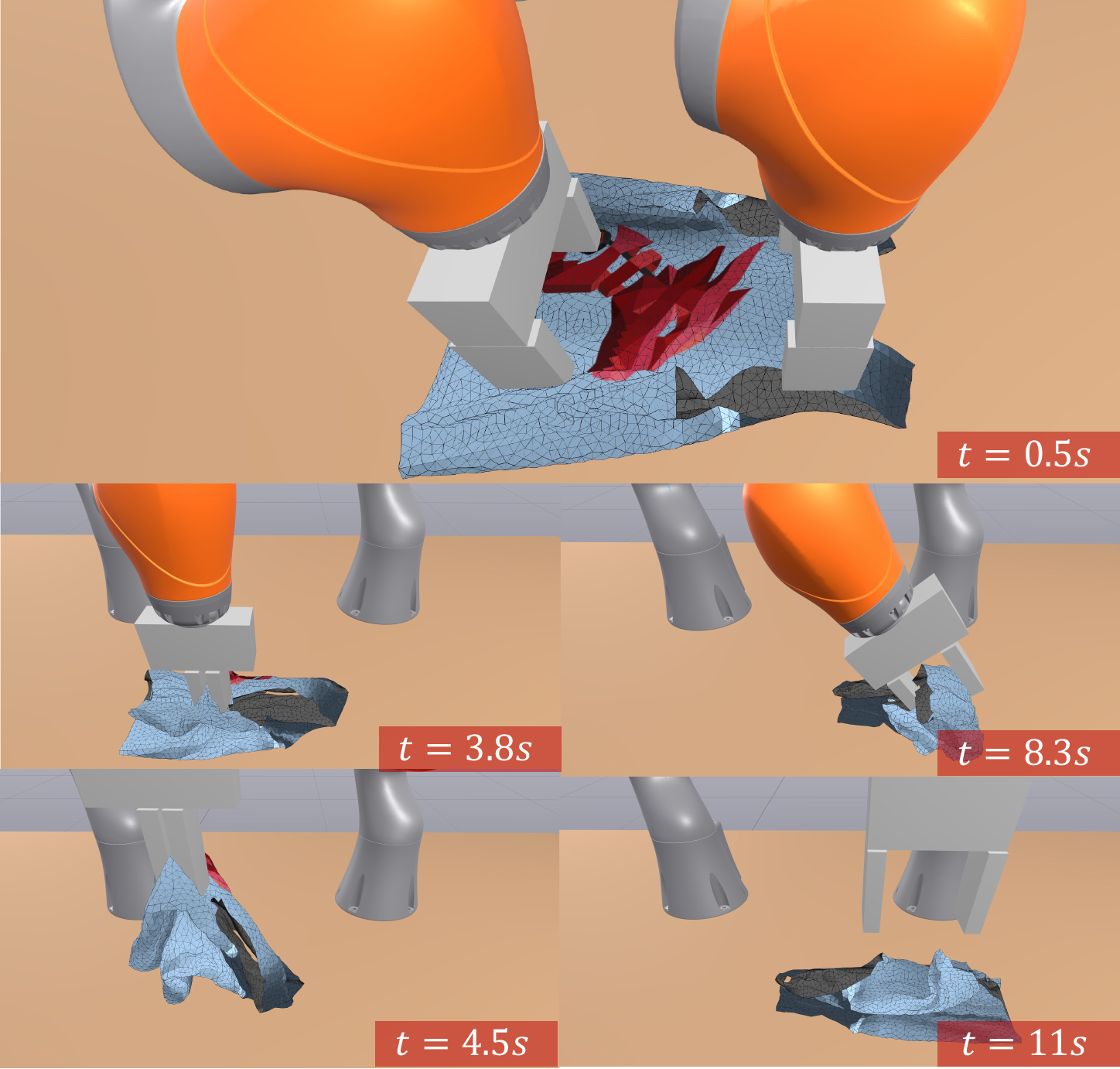}}
\caption{A challenging T-shirt folding task demonstrating the accuracy, speed, and robustness of our method. Two robotic arms fold and unfold the T-shirt while our approach accurately captures frictional interactions, effectively handling MPM-based cloth self-collisions and rigid-MPM coupling without penetration. See the supplemental video for details.}
\label{fig:dual_arm_folding}
\end{figure}
\section{PREVIOUS WORK}

\subsection{Physics-based Deformable Body Simulation}
Physics-based deformable body simulation has been a fundamental topic in computer graphics, computational physics, and engineering for decades, with applications ranging from material design to visual effects and virtual surgery. The Finite Element Method (FEM) has been widely used for simulating elastic and plastic deformations due to its accuracy in solving continuum mechanics problems \cite{bib:sifakis2012fem, bib:sofa2012}. However, FEM often requires high-quality meshes \cite{bib:madier2023introduction} and struggles with large deformations and topological changes \cite{bib:o1999graphical}. Additionally, handling contact among deformable bodies, including collision detection, requires significant computational effort \cite{bib:li2020ipc, bib:li2021cipc}.

To overcome these limitations, MPM has emerged as a powerful alternative due to its ability to handle large deformations, topological changes, and naturally resolve collision responses \cite{bib:sulsky1994particle, bib:jiang2016material}. It has been widely applied in simulations of granular materials \cite{bib:klar2016drucker}, volumetric and codimensional deformable bodies \cite{bib:jiang2017anisotropic, bib:han2019hybrid}, and fluid-solid coupling \cite{bib:gao2018animating}. Explicit integration in MPM is well-suited for massive parallelization on GPUs \cite{bib:fei2021principles, bib:gao2018gpu}. However, implicit integration remains computationally expensive \cite{bib:wang2020hierarchical}, especially when coupling MPM with rigid bodies for frictional contact \cite{bib:zong2024convex}. This coupling results in a large, often poorly-conditioned convex optimization problem, with no existing method efficiently solving it at interactive rates. To address this challenge, researchers have proposed asynchronous timestepping schemes, integrating MPM explicitly with smaller time steps to exploit massive parallelization while solving the remaining system at larger time steps \cite{bib:li2024dynamic, bib:gu2023maniskill2}. Our approach follows a similar asynchronous strategy but, to our knowledge, is the first to formulate the coupling between MPM and rigid bodies as an implicit convex optimization problem with a unique solution and global convergence guarantees, ensuring numerical stability.

\subsection{Frictional Contact}
Contact modeling in robotics is a broad topic beyond the scope of this paper; we refer readers to \cite{bib:le2024contact} for a comprehensive overview. Here, we focus on frictional contact involving deformable bodies, which is central to our work.

Many simulators, such as SOFA \cite{bib:sofa2012} and PhysX \cite{bib:macklin2019small}, formulate frictional contact as a linear complementarity problem (LCP) and solve it using projected Gauss-Seidel (PGS). Others address the nonlinear complementarity problem (NCP) directly with methods such as PGS \cite{bib:bullet}, non-smooth Newton solvers \cite{bib:macklin2019nonsmooth}, interior point methods \cite{bib:howell2022dojo}, or the alternating direction method of multipliers (ADMM) \cite{bib:menager2025differentiable}. While these approaches can model hard contact, they lack convergence guarantees, which can lead to robustness issues in practice. Incremental Potential Contact (IPC) \cite{bib:li2020ipc} offers strong non-penetration guarantees and demonstrates robust performance in practice \cite{bib:huang2021defgraspsim}. However, IPC loses convergence guarantees when friction is present, and its computational cost makes it less suitable for interactive simulations.

A different approach embraces compliant contact, which was first popularized in robotics by MuJoCo \cite{bib:todorov2014}. With a focus on physical accuracy, Drake employs introduces convex approximations of engineering-grade contact models in \cite{bib:castro2023theory}, later extended to support frictional contact between FEM and rigid bodies \cite{bib:han2023convex}. Moreover, Drake employs hydroelastic contact \cite{bib:elandt2019pressure, bib:masterjohn2022velocity, bib:castro2022unconstrained} to model continuous contact patches. More recently, Zong et al. \cite{bib:zong2024convex} extended this approach to couple MPM with rigid bodies, ensuring robustness and global convergence. However, this method is inherently serial and remains far from real-time performance. In this paper, we build upon \cite{bib:zong2024convex} and propose a GPU-accelerated approach that achieves interactive rates without compromising global convergence guarantees.

\subsection{Hardware Acceleration}
Researchers have explored two primary approaches to accelerating physics simulation using modern vectorized hardware like GPUs and TPUs. The first focuses on speeding up single large-scale simulations, enabling scenes with up to hundreds of millions of DoFs \cite{bib:hu2019taichi}. Significant progress has been made in accelerating FEM \cite{bib:xian2019scalable}, MPM \cite{bib:fei2021principles, bib:gao2018gpu, bib:wang2020massively}, and Discrete Element Method (DEM) \cite{bib:fang2021chrono}. However, some of these methods  remain far from real-time execution, while others do not support two-way coupling with rigid bodies.

The second approach emphasizes increasing simulation throughput, particularly for reinforcement learning and robotics training. Frameworks like MJX \cite{bib:mujoco3}, Isaac Gym \cite{bib:makoviychuk2021isaac}, Maniskill \cite{bib:tao2024maniskill3}, and Genesis \cite{bib:genesis} leverage hardware acceleration to run thousands of parallel simulations, improving data collection speed for policy learning.

Our GPU acceleration strategy aligns more with the first approach. However, rather than targeting massive simulations with millions of DoFs, we focus on interactive-rate simulations with thousands to tens of thousands of DoFs—expressive enough for robotic manipulation tasks. Unlike prior work prioritizing raw speed, our method emphasizes convergence and stability, ensuring reliable performance for robotics applications.
\section{OUTLINE AND NOVEL CONTRIBUTIONS}

In this work, we introduce several key contributions to advance the state of the art in efficient and robust deformable body simulation for robotics. First, we propose an asynchronous time-splitting scheme that enables integrating MPM and rigid body dynamics under significantly different time step sizes (Section \ref{sec:async}). Second, we develop a convex formulation for the weak coupling of MPM and rigid bodies through frictional contact, ensuring global convergence (Section \ref{sec:convex}). Third, we design a quasi-Newton solver that is well-suited for massive parallelization, enabling efficient execution on GPUs (Section \ref{sec:jacobi}). We detail the key implementation strategies required for high-performance execution on GPUs in Section \ref{sec:implementation}. We thoroughly validate our approach through a suite of robotic manipulation simulation scenarios and demonstrate the speed, accuracy, and robustness of our method in Section \ref{sec:results}. Finally, we release our implementation as part of the open-source robotics toolkit Drake \cite{bib:drake}, providing the robotics community with a scalable and reproducible solution for high-fidelity deformable body simulation.
\section{MATHEMATICAL FORMULATION}\label{sec:formulation}

\subsection{Asynchronous Time-Splitting Scheme}\label{sec:async}
We discretize time into intervals of size $\Delta t$ to advance the system dynamics from $t_n$ to the next time step $t_{n+1} := t_n + \Delta t$ using the governing equation:
\begin{equation}
    \mf{M}(\mf{q}) (\mf{v}^{n+1} - \mf{v}^n) = \Delta t\, \mf{k}(\mf{q}, \mf{v}) + \mf{\hat{J}}^T(\mf{q})\hat{\bgamma}(\mf{q}, \mf{v}), \label{eq:continuous}
\end{equation}
where $\mf{q}$ and $\mf{v}$ are generalized positions and velocities, $\bgamma$ is the contact and friction impulses, $\mf{k}$ represents all non-contact forces, and $\mf{\hat{J}}$ is the contact Jacobian. Using subscripts $d$ for MPM DoFs and $r$ for rigid body DoFs, we define:
\begin{align}
\mf{q}^T &= [\mf{q}_r^T,\; \mf{q}_d^T], & \mf{v}^T &= [\mf{v}_r^T,\; \mf{v}_d^T], \nonumber \\
\mf{k}^T &= [\mf{k}_r^T,\; \mf{k}_d^T], & \mf{M} &= \text{diag}(\mf{M}_r,\; \mf{M}_d).
\end{align}
Specifically, $\mf{q}_d$ and $\mf{v}_d$ denote the MPM \emph{particle} positions and \emph{grid} velocities ; $\mf{k}_d$ includes elastic-plastic and external forces on MPM DoFs; $\mf{q}_r$ and $\mf{v}_r$ are the articulated rigid body joint positions and velocities; $\mf{k}_r$ captures Coriolis terms and gravity on rigid body DoFs.

We further separate contact between MPM and rigid bodies from contact among rigid bodies as:
\[
\hat{\bgamma} = \begin{pmatrix} \bgamma_{d} \\ \bgamma_{r} \end{pmatrix}, \quad \mf{\hat{J}} = \begin{pmatrix} \mf{J}_r & \mf{J}_d \\ \mf{J}_{rr} & \mf{0} \end{pmatrix}.
\]
Notably, we do not explicitly model contact forces among MPM DoFs, as these are naturally resolved through the hybrid Eulerian/Lagrangian representation with constitutive and plasticity modeling \cite{bib:han2019hybrid, bib:jiang2017anisotropic}.  

To achieve efficient GPU parallelization without sacrificing stability, we integrate elastic forces explicitly using symplectic Euler, while treating stiff frictional contact forces implicitly via backward Euler. However, symplectic Euler requires smaller time step sizes than backward Euler for stability, motivating our asynchronous time-splitting approach:
\begin{align}
\mf{M}_r(\mf{q}_r^n)(\mf{v}_r^{n+1} - \mf{v}_r^{n}) 
    &= \Delta t\, \mf{k}_r(\mf{q}_r^n, \mf{v}_r^n) 
    + \mf{J}_{rr}^T \bgamma_{r}(\mf{v}_r^{n+1}) \nonumber \\
    &\quad {} + \mf{J}_r^T \sum_{k=0}^{N-1} \bgamma_{d}^{n,k}, \label{eq:discrete_rigid}
\end{align}
\begin{align}
\mf{M}_d(\mf{q}_d^{n, k})(\mf{v}_d^{n, k+1} - \mf{v}_d^{n, k}) 
    &= \frac{\Delta t}{N}\, \mf{k}_d(\mf{q}_d^n) 
    + \mf{J}_d^T \bgamma_{d}^{n,k}, \label{eq:discrete_mpm}
\end{align}
\[
\text{for } k=0,\ldots, N-1,\; \text{with } \mf{v}_d^{n,0}=\mf{v}_d^n,\; \text{and } \mf{v}_d^{n,N} = \mf{v}_d^{n+1}.
\]
Here, $\bgamma_{d}^{n,k} = \bgamma_{d}(\mf{v}_r^n, \mf{v}_d^{n, k+1})$ is the frictional impulse between rigid bodies and MPM at substep $k$. Notice that in Eq.~\eqref{eq:discrete_mpm}, we decompose a single time step into $N$ substeps. Elastic forces $\mf{k}_d$ are integrated explicitly to facilitate parallelization, whereas contact forces $\bgamma_{d}$ are integrated implicitly with respect to $\mf{v}_d^{n, k+1}$ for stability. Throughout all substeps in a single timestep, the rigid body velocity is fixed at $\mf{v}_r^n$. 
The contact impulses at each substep are accumulated and applied to rigid body DoFs in Eq.~\eqref{eq:discrete_rigid}. This approach results in a weak coupling between rigid bodies and MPM DoFs, in contrast to the strong coupling scheme proposed by~\cite{bib:zong2024convex}. We evaluate the accuracy of our weak coupling scheme in Section~\ref{sec:results}.

\subsection{Convex Formulation}\label{sec:convex}
The integration of rigid DoFs in Eq.~\eqref{eq:discrete_rigid} follows the same methodology as described in \cite{bib:castro2023theory}. We refer readers to that work for detailed explanations and  implementation practices. Here, we focus on the integration of Eq.~\eqref{eq:discrete_mpm} for MPM DoFs.
For notational simplicity, we drop the subscript $d$ and the superscript $n$ in Eq.~\eqref{eq:discrete_mpm}, reducing it to an algebraic difference equation that advances a substep of MPM:
\begin{equation}
    \mf{M}(\mf{v}^{k+1} - \mf{v}^k) = \frac{\Delta t}{N}\mf{k} + \mf{J}^T\bgamma(\mf{v}_c). \label{eq:discretet_mpm2}
\end{equation}
Here, $\mf{v}_c = \mf{J} \mf{v}^{k+1} + \mf{b}_r$ denotes the relative contact velocity between particles and the rigid body they are in contact with, expressed in the contact frame, where the contact frame is a local frame with the $z$-axis aligned with the contact normal. The term $\mf{b}_r = \mf{J}_r \mf{v}_r^n$ represents the bias velocity from the rigid body in the contact frame.

We solve Eq.~\eqref{eq:discretet_mpm2} in two stages by performing another time-splitting. First, we compute the \emph{free motion velocity}, $\mf{v}^*$, which is the velocity the MPM DoFs would attain in the absence of contact forces:
\begin{equation}
    \mf{M}(\mf{v}^{*} - \mf{v}^k) = \frac{\Delta t}{N}\mf{k}. \label{eq:free_motion}
\end{equation}
This step is equivalent to a standard explicit MPM step. We adopt the moving least-squares formulation described in \cite{bib:hu2018moving} for this computation.

In the second stage, we compute the \emph{post-contact velocity} $\mf{v}^{k+1}$ by solving:
\begin{equation}
    \mf{M}(\mf{v}^{k+1} - \mf{v}^*) = \mf{J}^T\bgamma(\mf{v}_c). \label{eq:contact_solve}
\end{equation}
To ensure global convergence, we reformulate Eq.~\eqref{eq:contact_solve} as a convex optimization problem \cite{bib:castro2023theory}:
\begin{equation}
    \mf{v}^{k+1} = \argmin_{\mf{v}} \ell_p = \argmin_{\mf{v}} \frac{1}{2}\|\mf{v} - \mf{v}^*\|_{\mf{M}}^2 + \ell_c(\mf{v}_c). \label{eq:optimization_problem}
\end{equation}
The term $\ell_c(\mf{v}_c)$ is the contact potential energy, defined such that $\bgamma(\mf{v}_c)=-\partial \ell_c/\partial \mf{v}_c$. With mass lumping, $\mf{M}$ is diagonal and positive definite, ensuring that problem \eqref{eq:optimization_problem} is strongly convex as long as the frictional contact model is designed with a convex $\ell_c$. We adopt the \emph{lagged contact model} in \cite{bib:castro2023theory}.
\subsection{Quasi-Newton Solver}\label{sec:jacobi}
A key advantage of our weak coupling scheme is that, instead of requiring the Schur complement of the Jacobian of the \emph{momentum residual} as in \cite{bib:zong2024convex}—which is unfriendly to GPU parallelization—our formulation replaces it with the diagonal mass matrix, which is trivial to parallelize. However, the Hessian $\mf{\hat{H}}$ of $\ell_p$ still contains off-diagonal entries from the Hessian of $\ell_c$. Given the massive parallelization capabilities of GPUs, it is more efficient to trade additional solver iterations for better parallelization (see Section \ref{sec:results-dough}). Therefore, we adopt a quasi-Newton strategy by approximating $\mf{\hat{H}}$ with its $3\times3$ block diagonal counterpart $\mf{H}$ to avoid the inherently serial Cholesky factorization of $\mf{\hat{H}}$. Algorithm~\ref{alg:sap} summarizes our quasi-Newton solver, leveraging the convergence criteria from \cite{bib:castro2022unconstrained} and \cite{bib:zong2024convex} to ensure robustness against large mass ratios and enable fair comparison in Section~\ref{sec:results-dough}. The resulting $\bgamma$ is accumulated into Eq.~\eqref{eq:discrete_rigid}.

\algblockdefx{RepeatUntil}{EndRepeatUntil}
{\textbf{repeat until}}{}
\algnotext{EndRepeatUntil}
\begin{algorithm}[t]
  \caption{Quasi-Newton solver for problem~\eqref{eq:optimization_problem}}
  \label{alg:sap}
  \begin{algorithmic}[1]
  \State Initialize $\mf{v} \gets \mf{v}^k$ \RepeatUntil
  $~\Vert\tilde{\nabla}\ell_p\Vert < \varepsilon_a + \varepsilon_r\max(\Vert\tilde{\mf{p}}\Vert,\Vert\tilde{\mf{j}_c}\Vert)$,
  \State $\Delta\mf{v} =
        -\mf{H}^{-1}(\mf{v})\nabla_\mf{v}\ell_p(\mf{v})$ \label{op:Newton_iteration} 
  \State $\displaystyle \alpha^* =
        \argmin_{\alpha > 0} \ell_p(\mf{v} + \alpha \Delta\mf{v})$ \label{op:line_search}
  \State $\displaystyle \mf{v} = \mf{v} +
       \alpha^*\Delta\mf{v}$
  \EndRepeatUntil 
  \State\Return $\{\mf{v}$,
       $\bgamma=\frac{\partial \ell_c}{\partial \mf{v}_c}\big|_{\mf{v}{_c (\mf{v})}}\}$
  \end{algorithmic}
\end{algorithm}
We conclude this section by noting that Algorithm~\ref{alg:sap} is globally convergent with at least a linear rate, as proved in the Appendix. In practice, the algorithm can be efficiently parallelized on the GPU, as detailed in Section~\ref{sec:quasi_newton_impl}.
\section{GPU IMPLEMENTATION}\label{sec:implementation}

\subsection{GPU Optimization for the Free Motion Solver}
\label{sec:gpu_optimization}

The explicit MPM time integration procedure for the \emph{free motion} solve in Eq.~\eqref{eq:free_motion} involves three main operations: (1) transferring particle attributes to the Eulerian grid, (2) updating grid values at each grid node, and (3) transferring updated grid values back to particles. For a detailed overview of the MPM algorithm, we refer readers to \cite{bib:jiang2016material}.

The key to high-performance MPM lies in optimizing particle-to-grid transfer operators, where we use one GPU thread to perform the computation for a single particle, and addressing write conflicts that occur during parallelization when multiple particles write to the same grid node is very important for performance. Conflicts within a GPU warp are particularly detrimental to performance. There are two known strategies to mitigate such conflicts:
\begin{itemize}
    \item Randomization: Proposed by \cite{bib:hu2019taichi, bib:wang2020massively}, this approach randomly shuffles particles. By introducing randomness into particle order, the chances of particles from different warps writing to the same grid location are reduced, as the large number of particles and grid nodes makes such conflicts statistically rare.
    \item Warp-Level Reduction: Introduced by \cite{bib:gao2018gpu, bib:fei2021principles}, this approach uses warp-level reduction before atomically writing data to grid nodes. It requires sorting particles into cells to maintain the order of reduction, offering more controlled conflict management.
\end{itemize}

We adopt the warp-level reduction strategy combined with a partial sorting approach. When particles are fully sorted such that particles transferring to the same grid nodes are consecutive and grouped together, reductions are first independently performed within each particle group. This operation is efficiently implemented using warp-level CUDA intrinsic functions \cite{bib:gao2018gpu}. The thread associated with the first particle in each group then writes the reduced result back to global memory. This ensures that each group performs only one atomic operation, significantly reducing atomic operation overhead.
This method is particularly effective for interactive-rate simulations with small problem sizes, where the GPU's computational throughput cannot be fully utilized, and memory access becomes the bottleneck. However, sorting at each substep is time-consuming due to kernel launches and implicit host-device synchronizations. To address this, we leverage the spatial continuity of simulations to reuse sorting results from the previous time step. In practice, sorting is only needed once per time step, allowing all substeps within that time step to reuse the sorted data, forming our partial sorting strategy. We use a radix sort with 10-bit keys instead of the full bit-width to enhance efficiency. While this approach may lead to suboptimal particle ordering and cause particles in the same cell being distributed across multiple warps, resulting in more atomic operations, the performance gains from reduced sorting time outweigh this drawback.

\subsection{Parallel Quasi-Newton Solver Implementation}
\label{sec:quasi_newton_impl}
At the beginning of each substep, we register a contact point for each particle that is inside a rigid geometry, creating one contact point per particle-geometry pair. The contact velocity at the $p$-th contact point is given by  
\[
\vf{v}_{c,p} = \mf{J}_p \mf{v}_d + \vf{b}_p.
\]
The first term represents the velocity of particle $p$ expressed in the contact frame, which can be efficiently computed in parallel using a grid-to-particle kernel. The second term, the bias velocity of the rigid body evaluated at the contact point, is constant throughout all substeps within a single time step, allowing it to be precomputed and reused across all substeps.

Similarly, the gradient and Hessian of the objective function in problem~\eqref{eq:optimization_problem} can be efficiently computed in parallel:
\begin{align}
    \nabla_\mf{v} \ell_p &= \mf{M}(\mf{v} - \mf{v}^*) + \mf{J}^T \frac{\partial \ell_c}{\partial \mf{v}_c} , \label{eq:gradient} \\
    \mf{\hat{H}} &= \mf{M} + \mf{J}^T \frac{\partial^2 \ell_c}{\partial \mf{v}_c^2} \mf{J} \label{eq:hessian}.
\end{align}

The first terms involving the mass matrix $\mf{M}$ is trivially parallelized because, with mass lumping, they reduce to a simple sweep over all grid nodes. For the second terms, we first compute the local 3-dimensional gradient vector $\frac{\partial \ell_c}{\partial \mf{v}_c}$ and the $3 \times 3$ Hessian matrix $\frac{\partial^2 \ell_c}{\partial \mf{v}_c^2}$ for each contact point. We then execute a single particle-to-grid kernel for all particles in contact to accumulate their contributions, leveraging the fact that $\mf{J}$ is the MPM transfer operator that maps grid velocity in the world frame to the particle velocities in the contact frame. By only keeping the diagonal terms when computing $\mf{\hat{H}}$, we obtain the approximate Hessian $\mf{H}$ used in step~\ref{op:Newton_iteration} of Algorithm~\ref{alg:sap}. Since $\mf{H}$ is block diagonal, the linear system in step~\ref{op:Newton_iteration} decouples into independent $3 \times 3$ subsystems, which can be efficiently solved in parallel using Cholesky factorizations. 
After obtaining the search direction, we perform the line search in step~\ref{op:line_search} using a one-dimensional Newton's method that falls back to bisection when convergence stalls, as described in \cite{bib:castro2022unconstrained}. Similar to the computations in Eq.~\eqref{eq:gradient} and Eq.~\eqref{eq:hessian}, the objective function and its gradient and Hessian required by the one-dimensional Newton solver are efficiently computed through a single particle-to-grid transfer for the $\ell_c$ term and a simple parallel reduction over the grid for the $\mf{M}$ term. With this strategy, the entire substep, including solving Eq.~\eqref{eq:free_motion} and Eq.~\eqref{eq:contact_solve}, is efficiently parallelized and entirely executed on the GPU, with our optimization practice in Section \ref{sec:gpu_optimization}.
\begin{figure}
  \centering
  \includegraphics[width=1.0\columnwidth]{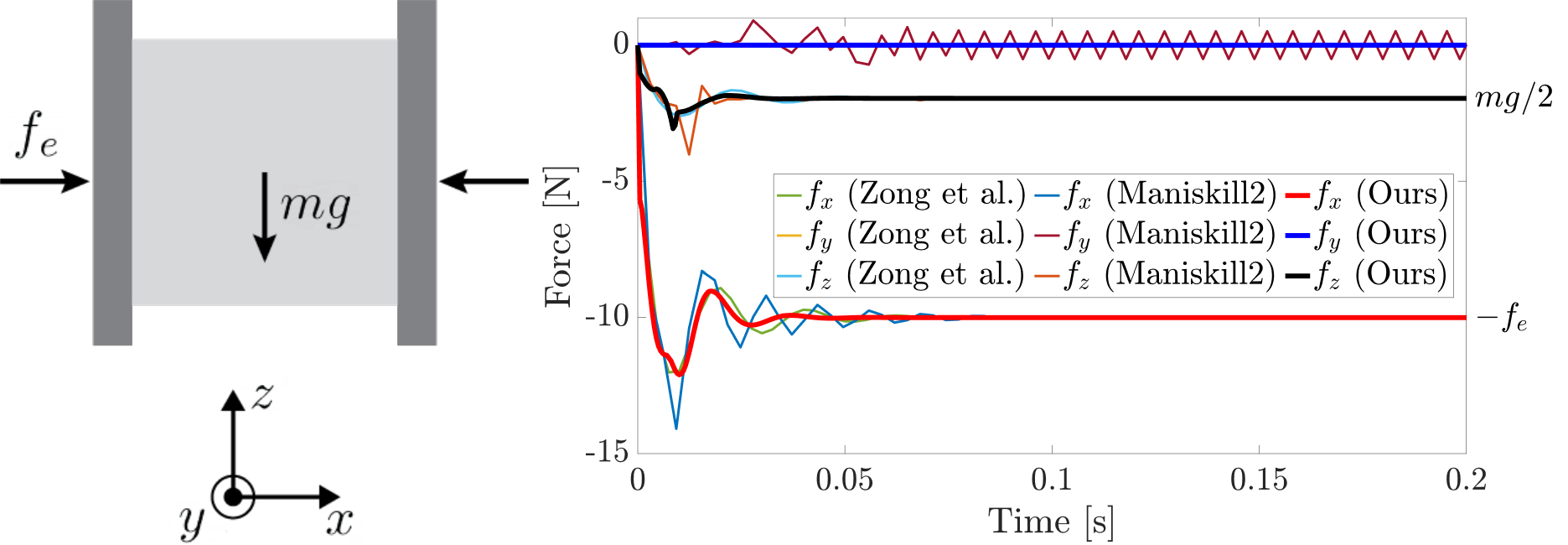}
\caption{(\textbf{Left}) Two rigid panels apply a normal force of $f_e = 10$~N in the $x$-direction to hold an elastic MPM box in place against gravity. (\textbf{Right}) The contact forces exerted by the elastic box on one of the rigid panels is compared using our method, ManiSkill2, and Zong et al. The corresponding analytical steady-state values are indicated on the right.}
\label{fig:gravity_comparison}
\end{figure}

\section{RESULTS}\label{sec:results}

We present several test cases to showcase the efficiency, accuracy, and robustness of our method. 
All simulations are run on a system with an Intel(R) Xeon(R) CPU E5-2690 v4 processor (56-core) and 128 GB of RAM, and an RTX 4090 with 24 GB device memory. We report scene statistics, including the Eulerian grid spacing $h$, MPM particles count, the time step size $\Delta t$, the number of substeps $N$, runtime performance, and number of contacts throughout the simulation in Table \ref{table:simulation}. All simulations are solved to convergence with $\varepsilon_r = 5\times10^{-2}$ and $\varepsilon_a$ set to machine epsilon unless otherwise specified.

\subsection{Comparison with Gu et al. \cite{bib:gu2023maniskill2} and Zong et al. \cite{bib:zong2024convex}}
\label{sec:results-force-plot}
We validate the accuracy of our convex formulation and parallel quasi-Newton solver in a comparison against the method adopted in ManiSkill2 \cite{bib:gu2023maniskill2}, a state-of-the-art embodied AI simulation environment with MPM support, and Zong et al. \cite{bib:zong2024convex}, which tightly couples MPM and rigid bodies using a convex formulation.

We replicate the experimental setup from \cite{bib:zong2024convex}, where an elastic cube is compressed by two rigid panels and held in place by friction, as shown in Fig.~\ref{fig:gravity_comparison}.
Our method is simulated with $\Delta t =0.1$ ms and $N=1$ substep, while Gu et al. uses $\Delta t = 0.01$ ms with 25 substeps for its explicit MPM integration, and Zong et al. employs $\Delta t = 10$ ms due to its fully implicit formulation.

The contact forces exerted by the cube on the left panel are shown in Fig.~\ref{fig:gravity_comparison}. The normal and $z$-direction friction forces converge to the analytical solutions after the initial transient period for all methods. Consistent with the observations in \cite{bib:zong2024convex}, the method by Gu et al. exhibits high-frequency oscillations in the $y$-direction friction force, as its explicit rigid-MPM coupling struggles with numerical stability in stiction.

Although our method also employs explicit integration for MPM elasticity forces, the contact coupling between rigid bodies and material points is as stable as in Zong et al., demonstrating a comparable level of accuracy to the fully implicit treatment. This stability stems from our convex formulation, which solves for contact implicitly. However, due to the weak coupling nature of our approach, our time step size is more limited compared to Zong et al., highlighting a limitation of our method. Despite this, our approach still achieves physically correct results, whereas explicit scheme by Gu et al. fails to produce the correct outcome, even with 10 times smaller time steps and 250 times more substeps.

To further stress test the robustness of our method, we replicate the challenging \emph{close-lift-shake} experiment from \cite{bib:zong2024convex}, where two rigid panels are controlled to grasp a rigid red cube positioned between two elastic cubes modeled with MPM (see Fig.~\ref{fig:shake}). This scenario serves as a stress test with pathological mass ratios using extreme parameters. The mass density of the MPM cube is $100~\text{kg/m}^3$, with a Young's modulus of $5 \times 10^5$ Pa and a Poisson's ratio of $0.4$, while the rigid cube has a mass density of $15000~\text{kg/m}^3$. Even in this extreme case, our method successfully completes the task with a time step size of $\Delta t = 0.1$ ms and a moderate tolerance of $\varepsilon_r = 10^{-2}$. Here, we increase the number of substeps to $N=10$ to maintain stability in the explicit MPM free motion solve, given the high material stiffness relative to its mass density.

\begin{figure}[b]
\centerline{\includegraphics[width=1.0\columnwidth]{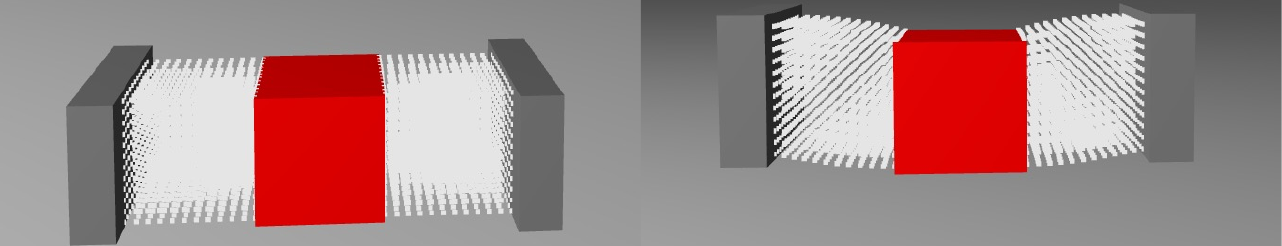}}
\caption{Our method successfully completes the stress test proposed in \cite{bib:zong2024convex}, which involves pathological mass ratios. The rigid panels securely hold a stack of three boxes with vastly different mass densities in place with friction, maintaining a stable grasp even under vigorous shaking.}
\label{fig:shake}
\end{figure}

\begin{figure}
\centerline{\includegraphics[width=1.0\columnwidth]{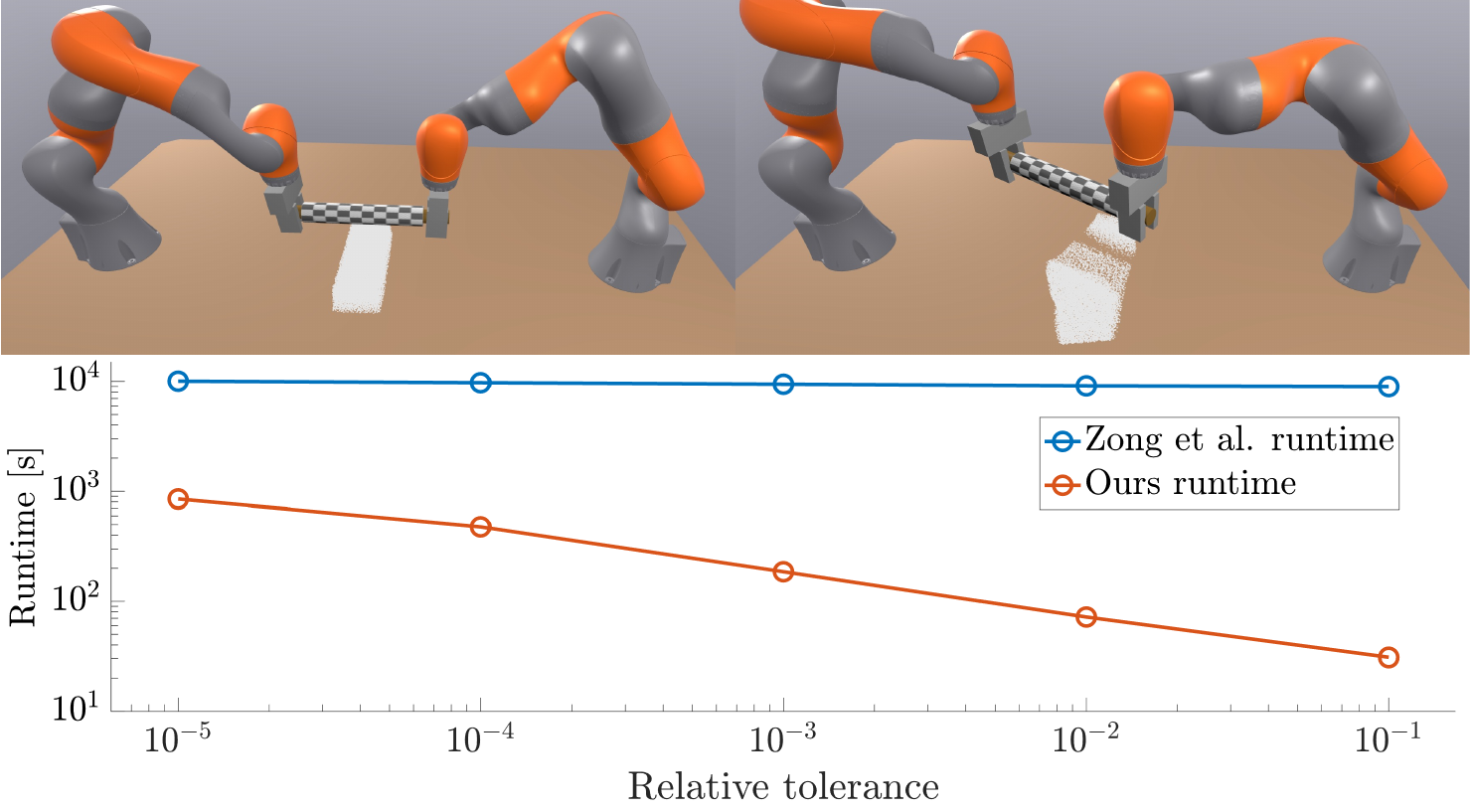}}
\caption{(\textbf{Top}) Flattening a piece of dough with a rolling pin. Our coupling scheme captures the rolling pin’s friction-driven rotation and the dough’s deformation. (\textbf{Bottom}) We compare the runtime between ours and Zong et al. \cite{bib:zong2024convex} under various relative tolerances.}
\label{fig:roll}
\end{figure}

\subsection{Rolling an Elastoplastic Dough}
\label{sec:results-dough}
We reproduce the complex dough rolling example from \cite{bib:zong2024convex} with our asynchronous time-splitting scheme, as shown in Fig.~\ref{fig:roll}. The friction coefficient is $1.0$ between the dough and the
rolling pin. The material parameters and prescribed motion trajectories are identical to those in \cite{bib:zong2024convex}. We refer readers to the supplemental video for the full task trajectory.

We compare the runtime performance of our method with \cite{bib:zong2024convex} under various tolerance criteria in Fig.~\ref{fig:roll}. 
For both methods, we set the absolute tolerance $\varepsilon_a$ in Algorithm~\ref{alg:sap} to machine epsilon and vary the relative tolerance $\varepsilon_r$ from $10^{-5}$ to $10^{-1}$. 
Both methods are simulated with a time step size $\Delta t=10$~ms, and our method uses $N=10$ substeps. Across all tolerance levels, our method consistently outperforms \cite{bib:zong2024convex}, achieving at least a 10× speed-up and demonstrating strong scalability as the convergence criteria are relaxed, making it well-suited for interactive-rate simulations. With  $\varepsilon_r = 10^{-1}$, our method attains a 500× speed-up compared to \cite{bib:zong2024convex}, reaching a 59\% real-time rate (defined as simulation time divided by wall-clock time). In practice, we observe that the simulation dynamics are insensitive to the relative tolerance $\varepsilon_r$. We provide a side-by-side comparison of the simulation runs with $\varepsilon_r = 10^{-1}$ and $\varepsilon_r = 10^{-5}$ in the supplementary video. Based on these observations, we generally use a loose relative tolerance of $\varepsilon_r = 5 \times 10^{-2}$ to maintain interactive-rate performance.

\begin{figure}[b]
\centerline{\includegraphics[width=1.0\columnwidth]{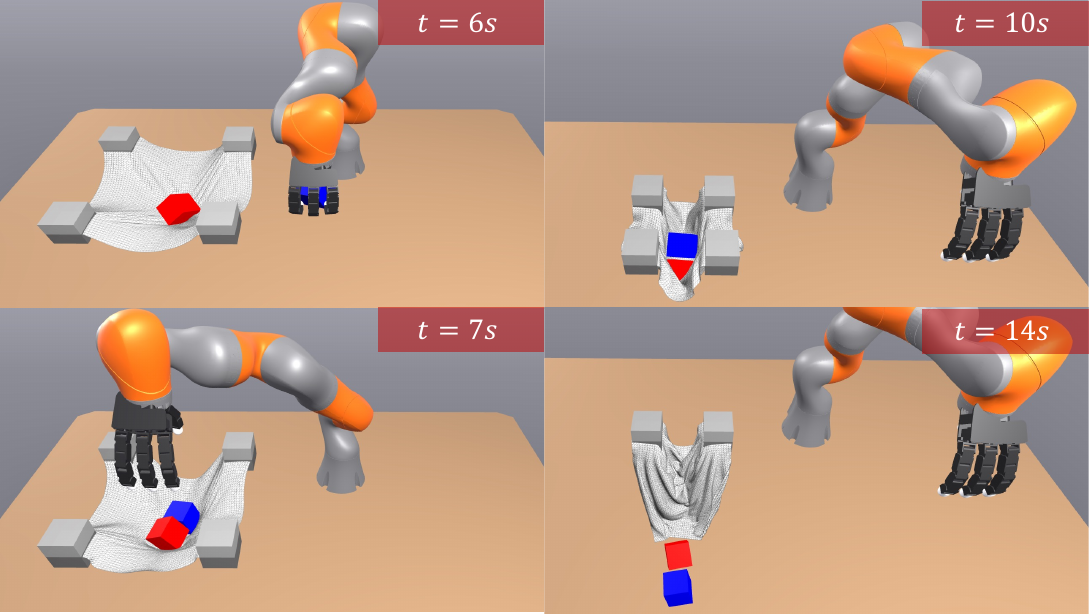}}
\caption{Pick and place rigid boxes into a deformable cradle.}
\label{fig:allegro_bagging}
\end{figure}

\subsection{Cloth and rigid bodies}
\label{sec:results-bagging}
Another key advantage of our weak coupling scheme is that it eliminates the requirement for a convex energy density in the constitutive model of MPM, as mandated by \cite{bib:zong2024convex} to maintain convexity in the optimization problem \eqref{eq:optimization_problem}. This flexibility unlocks the potential of MPM to model a broader range of materials.

In this experiment, we implement the method from \cite{bib:jiang2015affine} to model cloth with MPM. The frictional contact and self-collision of the cloth are automatically handled by the MPM grid. The simulation setup involves a $42~\text{cm}^2$ cloth modeled with a Young's modulus of $E_{\text{cloth}} = 3.2 \times 10^6$ Pa, a Poisson's ratio of $\nu = 0.4$, and a density of $\rho_\text{cloth} = 1.5 \times 10^3~\text{kg/m}^3$. The four corners of the cloth are fixed with boundary conditions to form a cradle. A KUKA LBR iiwa 7 arm, equipped with an anthropomorphic Allegro hand, picks up two rigid boxes and places them on the cloth. Each box has a side length of 8 cm and a density of $\rho_\text{box} = 10^3$~\text{kg/m}³. The friction coefficient is $0.2$ between the boxes and the cloth. The cloth corners are then moved inward to wrap around the rigid boxes, and finally, two corners are released, allowing the boxes to fall out. We refer readers to the supplemental video for the full trajectory.

\begin{figure}
\centerline{\includegraphics[width=1.0\columnwidth]{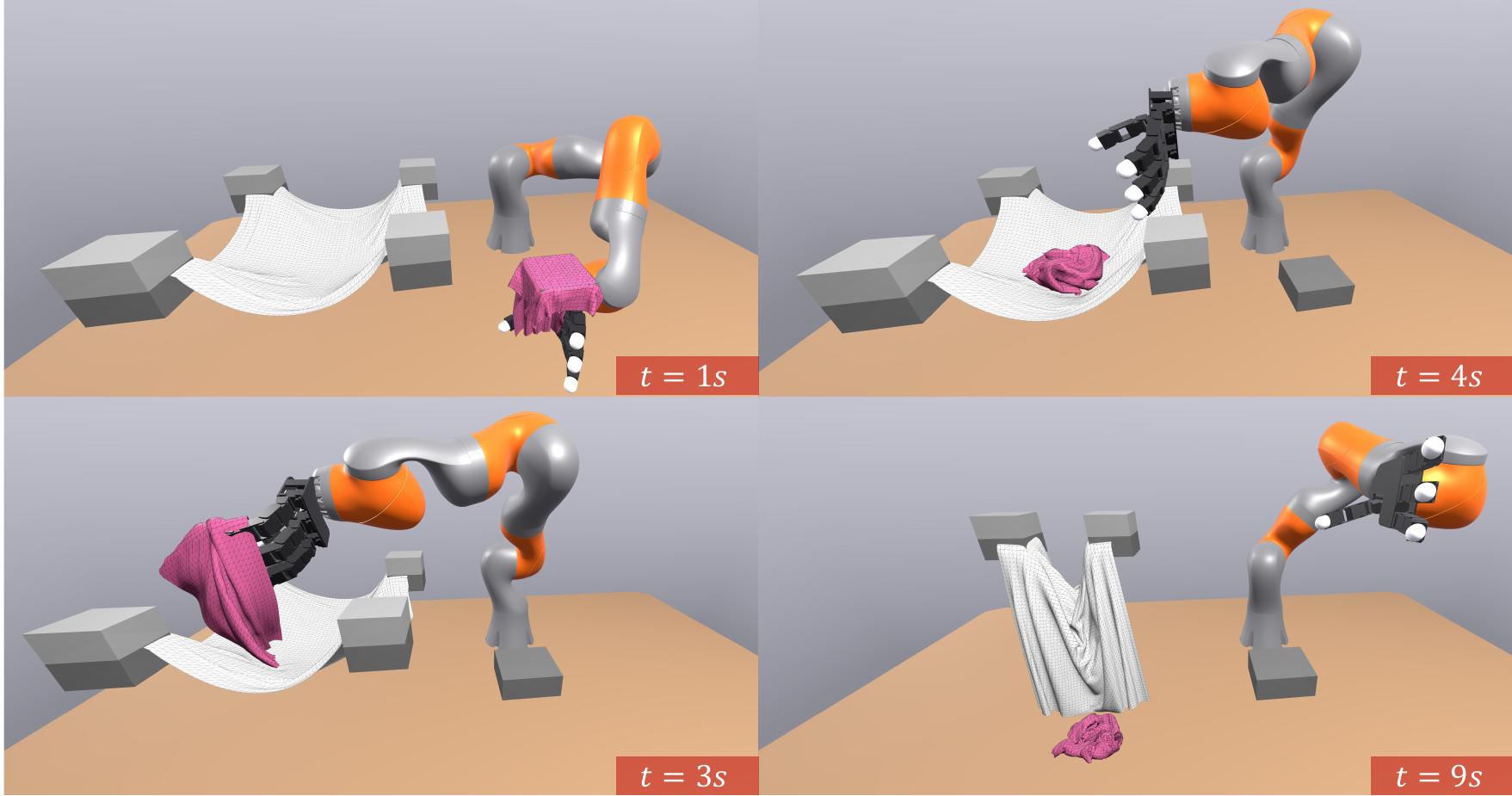}}
\caption{A robot picks up a piece of cloth and throws it in a ``laundry bag" . Our method resolves the complex cloth-on-cloth collisions and cloth self-collisions at interactive rate.}
\label{fig:allegro_bagging_cloth}
\end{figure}

\subsection{Laundry}
\label{sec:results-bagging-cloth}

Similar to the setup in Section~\ref{sec:results-bagging}, we replace the two boxes with another piece of cloth to mimic a common household scenario of moving clothes into a laundry bag. The robot follows a prescribed trajectory to grasp the cloth and throw it into the laundry bag.
We then release the grip on two corners of the laundry bag, allowing the cloth to fall out naturally. The simulation achieves a real-time rate of $26.5\%$, demonstrating that our method efficiently handles intense cloth-on-cloth collisions and self-collision scenarios.

\subsection{Folding and Unfolding a T-Shirt}
\label{sec:results-folding}

We demonstrate the accuracy and robustness of our method with a challenging T-shirt folding task (Fig.~\ref{fig:dual_arm_folding}). The T-shirt mesh consists of 4,171 vertices and 7,987 faces, and is simulated with a Young's modulus of $E_{\text{cloth}} = 10^5$~Pa, a Poisson's ratio of $\nu = 0.3$, and a density of $\rho_\text{cloth} = 10^3$~kg/m³. The robot setup features two PD-controlled KUKA LBR iiwa 7 arms equipped with custom parallel grippers. The friction coefficient between the grippers and the cloth is set as $0.6$. The simulation sequence proceeds as follows:
\begin{itemize} 
    \item The shirt free falls onto the table ($t = 0.5$ s).
    \item Both grippers perform the first fold along the long edge of the shirt($t = 3.0$ s), followed by a second fold along the short edge ($t = 3.8$ s, $t = 4.5$ s).
    \item One gripper grasps a corner of the shirt ($t = 8.3$ s), lifts it up vertically ($t = 10$ s), and unfolds it ($t = 11$ s).
    \item The grippers then fine-tune the shirt to fully flatten it.
\end{itemize}
For the full trajectory of this complex task, we refer readers to the supplemental video. Our method accurately captures frictional interactions between the gripper and the T-shirt, enabling smooth task execution. The T-shirt undergoes two folds, resulting in up to eight stacked layers. Our approach effectively handles cloth self-collisions and rigid-MPM interactions, preventing cloth self-penetration throughout the simulation. This ensures the T-shirt is successfully unfolded without artifacts, showing the robustness of our method.

\begin{table}
    \setlength{\tabcolsep}{2pt}
    \centering
    \fontsize{.8em}{.5em}\selectfont
    \caption{\textbf{Simulation statistics.} Runtime is measured as the simulation time in milliseconds per time step. Since the total number of DoFs in the system, $n_v$, varies as particles move through space, activating different grid nodes, we report the average $n_v$ over the entire simulation.
    We also provide the average and maximum number of contact points registered throughout the simulation as well as the total particle count. For the cloth simulation, we sample a particle at each mesh vertex and at the centroid of each triangle face, following the approach of \cite{bib:jiang2017anisotropic}.
    }
    \begin{tabular*}{\linewidth}{cccccccc}
        \toprule
        \multicolumn{1}{c}{Example} & $h\text{[m]}$ & \makecell{\# of\\Particles} &  $n_v$ & \makecell{$\Delta t$\\\text{[ms]}} & $N$ & \makecell{Runtime\\\text{[ms]}} & \makecell{Contacts \\Avg. (Max)}\\  \midrule
        Dough Rolling & 0.02 & 5,920 & 2,229.1 & 10 & 10 & $21.7$ & 17.8 (101) \\\midrule
        Box Bagging & 1/64 & 2,582 & 2,974.7 & 5 & 10 & $45.3$ & 19.6 (45) \\\midrule
        Laundry & 1/128 & 14,904 & 17,296.1 & 5 & 5 & $18.9$ & 163.4 (617)   \\\midrule
        T-shirt Folding & 1/128 & 12,160 & 8,299.6 & 2 & 4 & $21.5$ & 533.8 (1,307) \\\midrule
        Shake    & 0.01 & 3,456& 1,576.0 & 0.1 & 10 & $231.1$ & 571.1 (576)        \\\midrule
    \end{tabular*}
    \label{table:simulation}
\end{table}
\section{LIMITATIONS AND FUTURE WORK}
\textbf{Weak coupling:}
Our weak coupling approach enables broader material choices and efficient GPU parallelization, but it requires smaller time steps when there is a large mass ratio between rigid bodies and MPM bodies, as observed in Section~\ref{sec:results-force-plot}. Developing a strong coupling formulation that maintains parallelization efficiency for applications like high-frequency vibro-tactile feedback remains an active research direction for the authors.

\textbf{Explicit integration:} As a consequence of explicitly integrating the MPM free motion solve (Eq.~\eqref{eq:free_motion}), our method requires sufficient substeps to meet the stability criteria of explicit integration for MPM elastic forces. For example, in the stress test scenario described in Section~\ref{sec:results-force-plot}, we use substeps as small as $0.01$ ms to maintain stability. This contrasts with \cite{bib:zong2024convex}, where the scheme is unconditionally stable regardless of the time step size. However, we believe this trade-off of introducing an additional parameter to control substeps is justified by the significant performance gains, enabling interactive-rate simulations.

\textbf{Linear convergence:} As shown in Fig.~\ref{fig:roll}, our quasi-Newton solver exhibits linear convergence. Consequently, the speedup compared to \cite{bib:zong2024convex} is less pronounced when problem~\eqref{eq:optimization_problem} is solved to very tight tolerances, as \cite{bib:zong2024convex} uses Newton's method, which achieves quadratic convergence. However, as demonstrated in Section~\ref{sec:results-dough}, the dynamics of our simulation are largely insensitive to the tolerance $\varepsilon_r$, and we generally find that a loose tolerance of $\varepsilon_r = 5 \times 10^{-2}$ is sufficient for most manipulation tasks.

\textbf{Discrete contact detection:} Similar to \cite{bib:gu2023maniskill2} and \cite{bib:zong2024convex}, we employ discrete collision detection at each substep, where the rigid body's position is fixed at the beginning of the time step. This approach introduces the possibility that a rigid body may pass through a thin layer of MPM particles within a single time step without registering a contact. This effect is particularly pronounced in cloth simulations, where penetration can significantly alter subsequent dynamics. To mitigate this issue, we use a smaller time step size of 2 ms in Section~\ref{sec:results-folding}. Developing a more principled solution to this problem, such as incorporating continuous collision detection (CCD), remains an avenue for future work.

\appendix \label{sec:appendix}
We show that Algorithm~\ref{alg:sap} converges globally with at least linear rate. This follows immediately from the next lemma.

\begin{lemma}\label{lem:global-convergence}
Let $f : \mathbb{R}^n \to \mathbb{R}$ be $\mu$-strongly convex, and consider the quasi-Newton iterations with Hessian approximations $\mf{H}(\mf{v})$ performed using exact line search from an initial point $\mf{v}_0$. Define the sublevel set
\[
S(\mf{v}_0) = \{ \mf{v} \in \mathbb{R}^n : f(\mf{v}) \le f(\mf{v}_0) \}.
\]
Assume that for all $\mf{v} \in S(\mf{v}_0)$ the approximation $\mf{H}(\mf{v})$ is symmetric positive definite (SPD) with condition number bounded above by $\sigma$, and that $\nabla f$ is locally Lipschitz with Lipschitz constant $L$. Then the iterates converge to the unique minimizer $\mf{v}_*$ of $f$, with the error bound
\begin{equation}
    f(\mf{v}_m) - f(\mf{v}_*) \le \Bigl(1 - \frac{\mu}{\sigma^2 L}\Bigr)^m \Bigl[f(\mf{v}_0) - f(\mf{v}_*)\Bigr] \label{eq:linear_convergence_f}
\end{equation}
for all iterations $m \ge 0$. Furthermore,
\begin{equation}
    \|\mf{v}_m - \mf{v}_*\| \le C \rho^m \label{eq:linear_convergence_v}
\end{equation}
for some $C >0$ and $0 < \rho < 1$.
\end{lemma}

\begin{proof}
The proof for \eqref{eq:linear_convergence_f} can be found in Appendix E of \cite{bib:castro2022unconstrained}, where the result is derived using the Polyak inequality. Using $\mu$-strong convexity of $f$ around $\mf{v}_*$, we get 
\[\|\mf{v}_m - \mf{v}_*\|^2 \le \frac{\mu}{2}\Bigl[f(\mf{v}_m) - f(\mf{v}_*)\Bigr].\]
Taking square roots and applying \eqref{eq:linear_convergence_f} completes the proof.
\end{proof}
\begin{theorem}
Algorithm~\ref{alg:sap} is globally convergent with at least a linear rate.
\end{theorem}

\begin{proof}
Recall the objective function in problem \eqref{eq:optimization_problem}
\[
\ell_p(\mf{v}) = \frac{1}{2}\|\mf{v} - \mf{v}^*\|_{\mf{M}}^2 + \ell_c(\mf{v}_c).
\]
$\nabla \ell_p$ is locally Lipschitz since $\ell_p$ is continuously differentiable. In addition,
Since $\mf{M}$ is diagonal and SPD and $\ell_c$ is convex, the Hessian $\mf{\hat{H}}$ of $\ell_p$ is SPD. Therefore, its block-diagonal approximation $\mf{H}$ is also SPD. Thus, by applying Lemma~\ref{lem:global-convergence}, the result follows.
\end{proof}

\bibliographystyle{IEEEtran}
\bibliography{root}
\end{document}